\documentclass[a4paper,11pt]{article}
\usepackage{fullpage} 

\usepackage{diagbox} 
\usepackage{slashbox} 
\usepackage{tabularx} 
\usepackage{bbm} 
\usepackage{pifont}  
\usepackage{booktabs} 

\usepackage{amsmath}
\usepackage{amssymb}
\usepackage{amsthm}
\usepackage{mathtools}
\usepackage{mathdots}
\usepackage{appendix}
\usepackage{graphicx}
\usepackage{enumerate}
\usepackage{microtype}
\usepackage{mleftright}
\usepackage{mdframed}
\usepackage{authblk}
\usepackage{tcolorbox}
\usepackage[english]{babel}
\usepackage[utf8]{inputenc}
\usepackage{algorithm}
\usepackage[noend]{algpseudocode}
\usepackage{xspace}
\usepackage{tikz}
\usepackage{pgfplots}
\pgfplotsset{width=8cm,compat=1.9}
\usepackage{mathdots}
\usepackage{url}
\usepackage{forest}
\forestset{default preamble={for tree={s sep+=1cm}}}
\usepackage{caption,subcaption}
\usepackage[noadjust]{cite}
\usepackage{hyperref}

\tcbset{
    rounded corners,
    colback = white,
    before skip = 0.2cm,
    after skip = 0.5cm,
    boxrule = 1.5pt,
    arc = 5pt
}

\makeatletter
\renewcommand{\ALG@name}{Protocol}
\makeatother
\allowdisplaybreaks

\graphicspath{ {./figures/} }

\usepackage{algorithm}
\usepackage[noend]{algpseudocode}

\newtheorem{theorem}{Theorem}[section]
\newtheorem*{theorem*}{Theorem}

\newtheorem{lemma}[theorem]{Lemma}

\newtheorem*{corollary*}{Corollary}

\theoremstyle{definition}

\newcommand{\Ber}{\mathrm{Ber}}
\newcommand{\KL}{\mathop{D_{\mathrm{KL}}}\limits}

\newcommand{\EXP}{\mathsf{Exp}}

\DeclareMathOperator*{\argmax}{arg\,max}

\newif\ifanonymous
\anonymousfalse

\begin{document}
\title{A Tight Lower Bound for Non-stochastic Multi-armed Bandits with Expert Advice}

\author[1]{Zachary Chase}
\author[2]{Shinji Ito}
\author[3]{Idan Mehalel}
\affil[1]{Kent State University}
\affil[2]{The University of Tokyo and RIKEN}
\affil[3]{The Hebrew University}

\maketitle

\begin{abstract}
We determine the minimax optimal expected regret in the classic \emph{non-stochastic multi-armed bandit with expert advice} problem, by proving a lower bound that matches the upper bound of \cite{kale2014multiarmed}. The two bounds determine the minimax optimal expected regret to be $\Theta\mleft( \sqrt{T K \log \frac{N}{K} } \mright)$, where $K$ is the number of arms, $N$ is the number of experts, and $T$ is the time horizon.
\end{abstract}


\section{Introduction}
The seminal work \cite{auer2002nonstochastic} presented the fundamental \emph{non-stochastic multi-armed bandit with expert advice} problem (BwE, for short), which is defined as follows. Let $T,N,K \in \mathbb{N}$ where $K$ is the number of arms, $N$ is the number of experts, and $T$ is the time horizon.  On each round $t \in [T]$, each expert $j \in [N]$ selects an advice $e_t(j) \in [K]$ which is presented to the learner. Then, the learner pulls an arm $I_t \in [K]$ and the adversary simultaneously assigns a loss $\ell_t(r) \in \{0,1\}$ for each arm $r \in [K]$. The learner observes only the loss it suffers, namely $\ell_t(I_t)$, and the other losses remain unrevealed. Generally, the expert advice and losses are allowed to be fractional, that is, $e_t(j)$ is a probability distribution over $[K]$, and $\ell_t(r) \in [0,1]$. However, for the sake of proving the lower bound, the simpler discrete version of the problem suffices. Note that the adversary is \emph{adaptive}: the assignment of expert advice and losses to arms in each round may (and will, in our construction) depend on the learner's past choices. The learner's goal is to minimize the expected \emph{regret}: the difference between its cumulative loss and the cumulative loss of the best expert, which is the expert with minimal cumulative loss.

\subsection{Previous results}
Throughout the paper, we assume that $T \geq \Omega \mleft(K' \log (2N/K') \mright)$, where $K' = \min\{K,N\}$; in the complementing case, the lower bounds proved in this work and in \cite{auer2002nonstochastic} imply that the minimax optimal expected regret is $\Theta \mleft( T \mright)$. We further assume that $N \ge K$; in the complementing case, it is known that the minimax optimal expected regret is $\Theta\mleft(\sqrt{T N} \mright)$ \cite{auer2002nonstochastic, kale2014multiarmed}.

The seminal work \cite{auer2002nonstochastic} came up with the $\EXP4$ algorithm, which has been used as a building block or inspiration for many partial feedback algorithms. A partial list of examples include \cite{mcmahan2009tighter, seldin2011pac, daniely2013price, raman2024multiclass}. The analysis of \cite{auer2002nonstochastic} bounds the \emph{pseudo-regret} (see Section~\ref{sec:reduction-step1} for a formal definition) of $\EXP4$ by $O\mleft( \sqrt{T K \log N} \mright)$. This upper bound remained the best known for this problem, until the work \cite{kale2014multiarmed} which considered a generalization of the problem where the expert advice is only partially observed by the learner, which, as a by-product, improved the upper bound in the general case as well to $O\mleft( \sqrt{T K \log (N/K)} \mright)$. Furthermore, this quantity upper bounds the optimal expected regret of the problem, not just the weaker pseudo-regret. 

As for lower bounds, until the work \cite{seldin2016lower}, the best known lower bound was $\Omega\mleft( \sqrt{T} \mleft( \sqrt{K} + \sqrt{\log N} \mright)  \mright)$, where the left summand is from the lower bound for the multi-armed bandit (without experts) problem given in \cite{auer2002nonstochastic}, and the right summand is from the lower bound on the full-info prediction with expert advice problem, given e.g.\ in \cite{cesa2006prediction}. The work \cite{seldin2016lower} significantly improved the lower bound to $ \Omega \mleft(\sqrt{T K \frac{\log N}{\log K}} \mright)$, only off by a $\log K$ factor from the upper bound. They conjectured that this lower bound is tight, which is refuted in this work.

The works \cite{mannor2004sample, chen2024interpolating} identified a useful connection between minimizing regret problems (such as BwE), and identification of best arm/expert problems. Such identification problems are often referred to as the \emph{pure-exploration} variations of the regret minimization problems \cite{even2002pac, mannor2004sample}. Intuitively, the idea is that in settings where a single entity (arm/expert) is significantly better than all other entities throughout the game, a learner who minimizes regret must identify this entity in an early stage of the game. This idea allows us to consider identification problems instead of regret minimization problems.
Armed with this approach, and inspired by a problem instance presented in \cite{chen2024interpolating}, the two recent works \cite{ito2024minimax, cesa2025improved} have managed to prove a tight lower bound of $\Omega \mleft(\sqrt{T K\log (N/K)} \mright)$ against a restricted learner, having the property that its choice in round $t$ is independent of the expert advice of round $t$ (the result of \cite{cesa2025improved} is somewhat weaker, as they consider a slightly even more restricted learner). This setting is sometimes referred to as \emph{proper} online learning \cite{hanneke2021online}, while the general setting is known as \emph{improper} learning.

\subsection{Our contribution}
Building on the ideas and the hard problem instance  used in \cite{ito2024minimax}, we manage to prove the same $\Omega \mleft(\sqrt{T K\log (N/K)} \mright)$ lower bound for the classic problem, without placing any restrictions on the learner. The formal statement of the bound appears in Theorem~\ref{thm:main}. A summary of the bounds proved for this problem may be found in Table~\ref{tab:regretbound}. In the following subsection, we describe the proof sketch.

\begin{table}[t]
  \centering
  \begin{tabular}{lll}
    \toprule
    Setup & Reference & Bound \\
    \midrule
    Standard (pr) & \cite{auer2002nonstochastic} & $O\left(\sqrt{TK \log N}\right)$ \\
    Standard & \cite{kale2014multiarmed} & $O\left(\sqrt{TK \log \frac{N}{K}}\right)$ \\
    Standard & \cite{seldin2016lower} & $\Omega\left(\sqrt{TK \frac{\log N}{\log K}}\right)$ \\
    Proper* & \cite{cesa2025improved} & $\Omega \left(\sqrt{TK \log \frac{N}{K}}\right)$ \\
    Proper & \cite{ito2024minimax} & $\Omega \left(\sqrt{TK \log \frac{N}{K}}\right)$ \\
    Standard & \textbf{[This work] } & $\Omega \left(\sqrt{TK \log \frac{N}{K}}\right)$ \\
    \bottomrule
  \end{tabular}
    \caption{Upper ($O(\cdot)$) and lower ($\Omega(\cdot)$) bounds on the expected regret for the multi-armed bandit with expert advice (BwE) problem. The \emph{Standard (pr)} setup refers to the standard setup considered in this paper, but where the bounded quantity is the weaker pseudo-regret, and not the expected regret. The \emph{proper} setup refers to a restricted learner whose choice of arm in round $t$ does not depend on the advice of round $t$.  The \emph{proper*} setting refers to the setting considered in \cite{cesa2025improved}, in which the learner is slightly more restricted than in the proper setting.}
    \label{tab:regretbound}
\end{table}

\subsection{Proof Sketch}
Our proof proceeds in four steps, described in high-level in the following subsections. Many of the ideas we build on were discovered by \cite{ito2024minimax}, but there are a few key differences between the proofs, which we discuss in Section~\ref{sec:proof-differences}.
\subsubsection{Step 1: Reduction from \emph{Special Batch Identification} (SBI) to BwE} \label{sec:sketch-1}
The first idea used in the proof, originated from \cite{chen2024interpolating} and further used in \cite{ito2024minimax} is to reduce some sort of an identification problem to BwE. We call the specific identification problem we use \emph{Special Batch Identification} or SBI, for short. This problem is well defined in the setting we consider, in which we partition the $N$ experts to roughly\footnote{In the formal proof we use one more ``dummy" batch, which we add for technical reasons.} $k = K/2$ many batches, where each batch contains roughly $n = N/k$ experts. Every batch $u \in [k]$ is a distinct ``world" in the sense that the experts in batch $u \in [k]$ can only advise one of two arms associated with their batch: $(u,0)$ or $(u,1)$. Furthermore, in every round, precisely one arm of each batch is ``correct" (has loss $0$) and the other is incorrect (has loss $1$). This creates a situation of $k$ many distinct binary ``worlds". The adversary chooses in advance precisely one of those batches to be ``special". The special batch behaves slightly different from the other batches. In the standard batches, all experts have probability $1/2$ to choose the correct label in every round. The special batch, however, contains a special expert $e^\star$, for which the adversary makes sure that in every round, the arm advised by $e^\star$ is chosen to be correct with probability $1/2+\epsilon$, where $\epsilon$ is some small parameter. The point here is that a learner who wishes to minimize regret essentially competes with this special expert. Therefore, it must identify the special expert with high probability, and then repeatedly copy its advice. We show that the bottleneck of this process is the identification of the special batch. This gives the reduction from SBI to BwE in our setting.
\subsubsection{Step 2: The one-batch SBI game}
Intuitively, since every batch functions as a separate ``world", in order to find the special batch, the learner has no choice but to go through the batches one by one and query them for some time, until it finds the special batch. Therefore, the learner essentially must solve $\Omega(k)$ one-batch SBI instances in expectation, until it finds the special batch (In a one-batch SBI game, the learner should decide for a given batch, if it is special or not). 

Since in every batch, and in every round, precisely one arm is correct, the special case of a single batch is a full-information problem: the learner receives all relevant information in every round, no matter which of the two arms it picks. This makes the problem reasonable to analyze: we need to compare between the distributions on the input generated by the adversary when the batch is special, and when it is not special. This was done by \cite{ito2024minimax}, by bounding the KL-divergence between the two distributions. We use the same bound in our proof as well. This bound shows that as long as the learner sees the information produced by the adversary for fewer than $c \frac{\log n}{\epsilon^2}$ many rounds (for some constant $c$), the cases where the batch is special or not are essentially indistinguishable.

\subsubsection{Step 3: From the one-batch SBI game to the general SBI game} \label{sec:sketch-one-to-many}
As mentioned before, the adversary can force the learner to solve $\Omega(k)$ many one-batch instances until it identifies the special batch. Therefore, the learner must run for $\Omega\mleft( k \frac{\log n}{\epsilon^2} \mright)$ many rounds in order to identify the special batch with high probability.

\subsubsection{Step 4: From the general SBI game lower bound to a BwE lower bound}
Fix $K,N,T$. We assume that for the setting we have defined in Section~\ref{sec:sketch-1}, there exists a learner $A$ with expected regret at most $O(\epsilon T)$. By Step 1, this intuitively means that $A$ must have (with high probability) identified the special batch at some round $T' < T$, and it is imitating the predictions of the special expert since then. However, from Step 3, we know that to identify the special batch, $A$ must run for at least $T' = \Omega \mleft(k \frac{\log n}{\epsilon^2} \mright)$ many rounds. Therefore, for small enough $\epsilon$, we get $T' > T$, which means that the special batch is not identified, and every learner has expected regret at least $\Omega(\epsilon T)$. Choosing $\epsilon = \Theta\mleft( \sqrt{\frac{k \log n}{T}} \mright)$ is small enough to get $T' > T$, and the lower bound $\Omega(\epsilon T) = \Omega \mleft( \sqrt{T K \log \frac{N}{K}} \mright)$ is implied.


\subsection{Key differences from the proof of \cite{ito2024minimax}} \label{sec:proof-differences}
We prove the same lower bound that \cite{ito2024minimax} proved against a restricted learner whose predictions in rounds $t$ are independent of the expert advice of round $t$ (that is, against a proper learner). While our proof is based on the ideas of \cite{ito2024minimax}, there are a few key differences, allowing us to prove the same lower bound against a non-restricted learner. Below, we briefly describe the main differences between our proof and the proof of \cite{ito2024minimax}.
\begin{enumerate}
    \item To make Step 3 possible, we have to make sure that the $k$ ``one-batch" problems are ``identical but independent" even when the learner is improper (not restricted as in \cite{ito2024minimax}). This forces the learner to solve from scratch at least an $\Omega(k)$ one-batch problems in order to identify the special batch. To implement it, we adapt the expert advice to the learner's choices of arms. This is where we use the adaptiveness of the adversary, which is not used in \cite{ito2024minimax}. The formal implementation of the adversary is described in Section~\ref{sec:restricted-step0}. 
    \item In \cite{ito2024minimax}, the reduction to BwE (Step 1) is proved from the more standard problem of \emph{Best Expert Identification} (BEI), in which a specific special expert needs to be identified, not only the batch containing it, as in SBI. However, the adaptive adversary we have defined to make Step 3 feasible, makes such a reduction infeasible, and therefore we prove the reduction from the seemingly easier SBI problem, and for an improper learner. Nevertheless, we observe that SBI is still difficult enough, and the lower bound proved in \cite{ito2024minimax} for BEI holds for SBI as well. This is done in the formal proof of Step 2 (Section~\ref{sec:two-batch-step2}).
\end{enumerate}


\subsection{Rest of the paper organization}
The rest of the paper (except for Section~\ref{sec:future} which discusses future work), is dedicated to the formal proof. To avoid any confusion, we stress here that in the formal proof, we add another ``dummy batch", which is not mentioned in the proof sketch. Therefore, in the formal proof, the ``one-batch" instance of the problem mentioned in the above sketch is referred to as the ``two-batch" instance (including the dummy batch). The exact construction is provided in the following section.

\section{A reduced setting of BwE (Step 0)} \label{sec:restricted-step0}
To prove the lower bound, we consider only the following set of instances of the problem. We assume that $k= (K-1)/2 \in \mathbb{N}$, and that $n=(N-1)/k \in \mathbb{N}$. Note that for the lower bound that we prove, this assumption is without loss of generality: for other values of $N,K$, we may simply use only the maximal number of arms $K'$ and experts $N'$ that fits the above requirements, and let the other arms and experts always suffer loss $1$. This will result in the same bound as if $N,K$ fit the above requirements, up to a multiplicative constant.

Under the assumption above, we may partition $N-1$ of the experts to $k$ disjoint batches, each of size $n$. The experts in the $u$'th batch (where $u \in [k]$) will only set their advice to one of two arms: $(u,0)$ or $(u,1)$. Each of the $N-1$ experts in the $k$ batches is identified by a pair $(u,v)$ where $u\in[k]$ identifies the batch, and $v \in [n]$ identifies the index of the expert inside the batch. Note that we have one unused expert and one unused arm. We identify this expert with $0$ and this arm also with $0$. Expert $0$ can only set its advice to $0$.

Within this reduced setting, we only consider the following possible strategies of the adaptive adversary.
\paragraph{Expert advice assignment.}
We denote the advice in round $t$ (of all experts other than $0$) by a vector $A_t = (a_t^{(1)},...,a_t^{(k)})$, where $a_t^{(u)}$ is a binary vector of length $n$ setting the advice of batch $u$ in the natural way: the $v$'th index of $a_t^{(u)}$, denoted $a_t^{(u,v)}$ sets the right value of $e_t(u,v)$. In all the adversarial strategies that we consider, the expert advice $A_t$ of each round $t$ is chosen as follows.

\begin{enumerate}
    \item Initialize $A_t$ by choosing every entry iid from $\Ber(1/2)$.
    \item For round $t = 1,2, \ldots$:
    \begin{enumerate}
        \item Use the advice $A_t$ for the course of this round. 
        \item Upon the choice of $I_t$ by the learner, for each $u \in [k]$:
        \begin{enumerate}
            \item If $I_t \in \{(u,0),(u,1)\}$, draw a binary vector $b$ of length $n$, each entry iid from $\Ber(1/2)$, and set $a_{t+1}^{(u)} = b$.
            \item Otherwise, set $a_{t+1}^{(u)} = a_{t}^{(u)}$.
        \end{enumerate}
    \end{enumerate}
\end{enumerate}

In simple words, the advice of every batch $u \in [k]$ is initialized by a vector of entries drawn iid from $\Ber(1/2)$. Only after rounds where an arm of batch $u$ is pulled, new expert advice is picked for this batch $u$, again by a vector of entries drawn iid from $\Ber(1/2)$. After rounds in which an arm of $u$ is not pulled, the same advice is reused for the next round.

\paragraph{Loss assignments.}
For any fixed $0 < \epsilon \leq 0.1$, we define $N$ different randomized loss assignment strategies that depend on $\epsilon$. Formally, a strategy is a randomized algorithm that the adversary uses to draw the expert advice and the loss of each arm. In the reduced setting that we consider, the adversary may choose precisely one of those strategies to be used throughout the entire game. The learner does not know which strategy is chosen by the adversary. If an adversary follows strategy $S$, it is called an $S$-adversary.
Fix $0 < \epsilon \leq 0.1$. 
In all strategies we define, the loss of all arms is always taken from $\{0,1\}$. Furthermore, for any batch $u \in [k]$ and any round $t$, $\ell_t(r)$ equals $0$ for precisely one $r \in \{(u,0), (u,1)\}$. We refer to this arm as the \emph{correct arm} in  batch $u$ and round $t$.
Let us define the strategies. First, any of the strategies we define draws $\ell_t(0) \sim \Ber(1/2-\epsilon/2)$. We describe the differences between the strategies below.
\begin{itemize}
    \item The first strategy is called $S_0$. An $S_0$-adversary draws $\ell_t(u,0) \sim \Ber(1/2)$ for all $u \in [k],t \in [T]$. Note that this determines $\ell_t(u,1)$ for all $u \in [k]$ and all $t \in [T]$.
    \item The other $N-1$ strategies are $\{S_{(u,v)}\}_{u \in [k], v \in [n]}$. An $S_{(u^\star,v^\star)}$-adversary draws $\ell_t(u,0) \sim \Ber(1/2)$ for all $t$, and for all $u \in [k] \backslash \{u^\star\}$. It draws $\ell_t(u^\star,a_t^{(u^\star, v^\star)}) \sim \Ber(1/2-\epsilon)$.
\end{itemize}

In simple words, in $S_0$ all experts are equally handled, and $S_{(u^\star,v^\star)}$ gives a slight advantage to the predictions of $(u^\star,v^\star)$.

\section{A reduction from \emph{Special Batch Identification} (SBI) to BwE (Step 1)} \label{sec:reduction-step1}
We say that $u \in [k]$ is the \emph{special batch} if there exists $v \in [n]$ such that the adversarial strategy is $S_{u,v}$. If there is no such $u$, then $0$ is the special batch. Note that in our reduced setting, there is always precisely one special batch.

In this paper, we lower bound the \emph{pseudo-regret} of the learner (which can only be smaller than the expected regret):
\[
    R_T = \max_{j^\star \in [N]} \mathbb{E} \mleft[ \sum_{t=1}^T \ell_t(I_t) - \sum_{t=1}^T \ell_t(e_t(j^\star)) \mright].
\]
Further discussion on standard measures of regret in bandit problems may be found in \cite{bubeck2012regret}.
Note that from linearity of expectation, the pseudo-regret against any fixed expert $j \in [N]$ can be seen as a sum of $T$ \emph{per-round} pseudo-regret values, the $t^{th}$ of which is $\mathbb{E}[\ell_t(I_t) - \ell_t(e_t(j))]$.

In this section, we prove that in our reduced setting, any algorithm achieving pseudo-regret at most $\epsilon T /1000$, must identify the special batch with high probability. Note that this rather is intuitive: a learner who does not know which batch is the special one, and thus keeps pulling arms not from the special batch in at least, say, half of the rounds, will suffer at least $\epsilon/2$ per-round pseudo-regret against the best expert in at least half of the rounds, resulting in at least $\epsilon T /4$ cumulative pseudo-regret. 

We call this problem of identifying the special batch \emph{Special Batch Identification (SBI)}.
Formally, an algorithm $A'$ for SBI is an algorithm which is being executed in the same setting as BwE. The difference is that the goal of $A'$ is different from the goal of an algorithm for BwE. The goal of $A'$ is to identify the special batch with high certainty, after running for the least possible number of rounds. That is, in SBI, the algorithm may stop the game at any point, and it is required to output a prediction of the special batch's identity that is correct with high probability. In particular, it does not make any predictions or suffer any loss; it only chooses an arm to pull in each round, in a way that maximizes output accuracy, in the least possible number of rounds. Therefore, note that:
\begin{enumerate}
    \item Expert $0$ and arm $0$ behave the same under all strategies, and thus its per-round expected loss is known to always be $1/2-\epsilon/2$. That is, pulling $0$ does not provide any information required to estimate if $0$ is the special batch or not. We can thus assume w.l.o.g that $A'$ never pulls $0$.
    \item We can say that $A'$ ``pulls a batch" $u$ instead of a specific arm $(u,0)$ or $(u,1)$, as pulling $(u,0)$ or $(u,1)$ gives $A'$ the exact same information. Indeed, $\ell_t(u,0) = 1- \ell_t(u,1)$.
\end{enumerate}
We say that $A'$ is \emph{good} if for every possible strategy from our pool, $A'$ outputs the correct special batch with probability at least $0.95$. We denote by $T(A',S)$ the expected number of rounds $A'$ is executed before outputting its estimate, when the adversarial strategy is $S$. The expectation is taken over the randomness of $A'$, as well as the randomness of $S$.

\begin{lemma} \label{lem:reduction}
    Suppose that $A$ is an algorithm for BwE, such that for any $S$ from our pool of strategies, the pseudo-regret of $A$ is bounded by $R_T \leq r(T)$ for all $T$. Let $T^\star \geq 1000r(T^\star)/\epsilon$. Then, there exists a good algorithm $A'$ for SBI such that for any strategy $S$ from our pool, $T(A',S) \leq T^\star$.
\end{lemma}

\begin{proof}
    Let us define a good algorithm $A'$ with $T(A',S) \leq T^\star$ for all $S$. $A'$ executes $A$ for $T^\star$ many rounds. For any $u \in [k] \cup\{0\}$, let $T_u$ be the number of rounds where $A$ pulls an arm of batch $u$. $A'$ outputs $\argmax_{u \in [k] \cup \{0\}} T_u$. Clearly, $T(A',S) \leq T^\star$ for all $S$. It remains to show that $A'$ is good. Let $u^\star$ be the special batch. We show that $T_{u^\star}$ is likely to be very large compared to other $T_u$'s:
    \begin{align*}
        \Pr[T_{u^\star} < 0.75 T^\star] &= \Pr[T^\star - T_{u^\star} \geq 0.25T^\star] \\
                                        & \leq \frac{4 \mathbb{E}[T^\star - T_{u^\star}]}{T^\star} \tag{Markov}\\
                                        & \leq \frac{4}{T^\star} \cdot \frac{2 R_{T^\star}}{\epsilon} \tag{$R_{T^\star} \geq \frac{\epsilon}{2} \mathbb{E}[T^\star - T_{u^\star}]$} \\
                                        & \leq \frac{4}{T^\star} \cdot \frac{2 r(T^\star)}{\epsilon} \\
                                        &\leq  \frac{4}{T^\star} \cdot \frac{2 \frac{T^\star \epsilon}{1000}}{\epsilon} \\
                                        &\leq 0.01.
    \end{align*}
    To see that $R_{T^\star} \geq \frac{\epsilon}{2} \mathbb{E}[T^\star - T_{u^\star}]$, note that if the special batch is $0$, then the per-round pseudo-regret of any arm other than $0$ against expert $0$ is exactly $\epsilon/2$. If the special batch is some $u\in [k]$, then when pulling an arm of some other $u' \in [k]$ the per-round pseudo-regret against the special expert is exactly $\epsilon$, and when pulling $0$, the per-round pseudo-regret against the special expert is exactly $\epsilon/2$.

    Note that $T_{u^\star} \geq 0.75 T^\star$ implies $u^\star = \argmax_{u \in [k] \cup \{0\}} T_u$. Therefore, $A'$ outputs $u^\star$ with probability at least $0.99$, as required.
\end{proof}

\section{Lower bound for the two-batch game (Step 2)} \label{sec:two-batch-step2}
In this section we consider the case $K=3$, in which $k=1$, and we only have the batch containing only $0$, and another single batch. We prove a lower bound on $T(A', S_0)$, for any good algorithm $A'$ for SBI.
First, observe that in the case $k=1$, the learner is fixed to always choose batch $1$.
Therefore, for any adversarial strategy, the learner is completely passive throughout the game. Its only choice is for how many rounds to play, and which batch, $0$ or $1$, to output at the end of the game. Therefore, for any fixed strategy, the adversary draws the expert advice and arm losses from the same distribution in all rounds. For a strategy $S_\star$, let $P_\star$ be the distribution from which the adversary pulls the advice and losses in every round when the strategy is $S_\star$.
For any distribution $P$, let $P^T$ be $T$ iid draws from $P$.
So, if the adversary uses strategy $S_\star$ and it is given that the learner runs for $T$ many rounds, then the distribution over all expert advice and losses throughout the game is $P_{\star}^T$.
Since the learner depends only on the observed advice and losses, $P_{\star}^T$ induces a distribution on the learner's choice of good batch, when running for $T$ many rounds against an $S_\star$-adversary.
Let $S_{mix^T}$ be the adversarial strategy where the adversary draws in advance $v$ from $[n]$ uniformly at random, and then uses the strategy $S_{(1,v)}$ for $T$ many rounds. 
Likewise, denote $P_{mix^T} = \frac{1}{n} \sum_{v \in [n]} P_{(1,v)}^T$.
In all notation defined above, when $T$ is replaced with $\infty$, this is interpreted as a product of unbounded length. We denote the KL-divergence between two distributions $P,Q$ by $\KL(P || Q)$.

We begin with the following lemma, which was proved in \cite{ito2024minimax} for equivalent distributions that were defined slightly differently in \cite{ito2024minimax}. For completeness of the paper, we provide the proof below.
\begin{lemma}[\cite{ito2024minimax}] \label{lem:kl-bound}
    If $0 < \epsilon \leq 0.1$, then for any $T \geq 1, n \geq 1$:
    \[
    \KL(P_{mix^T} || P^T_0) \leq \frac{(1 + 4 \epsilon^2)^T -1}{n}.
    \]
\end{lemma}

\begin{proof}
    By definitions of $P_{mix^T}, P^T_0$ , we have
    \[
        \KL(P_{mix^T} || P^T_0) = \KL \mleft(\frac{1}{n} \sum_{v \in [n]} P_{(1,v)}^T || P^T_0 \mright)
    \]
    Let $p_v$ and $p_0$ be the probability mass functions for $P_{(1,v)}^T$ and $P^T_0$, respectively.  By definition of the KL-divergence and linearity of expectation, the above is equal to
    \begin{equation} \label{eq:kl-bound}
    \frac{1}{n} \sum_{v^\star \in [n]} \mathbb{E}_{g \sim P_{(1,v^\star)}^T} \mleft[ \ln \mleft( \frac{1}{n} \sum_{v \in [n] } \frac{p_v(g)}{p_0(g)} \mright) \mright]
    \leq
    \frac{1}{n} \sum_{v^\star \in [n]} \ln \mleft( \frac{1}{n} \sum_{v \in [n]}  \mathbb{E}_{g \sim P_{(1,v^\star)}^T} \mleft[ \frac{p_v(g)}{p_0(g)}  \mright] \mright),
    \end{equation}
    where the inequality is due to Jensen's inequality combined with concavity of $\ln$, and linearity of expectation. It remains to calculate the expectation $ \mathbb{E}_{g \sim P_{(1,v^\star)}^T} \mleft[ \frac{p_v(g)}{p_0(g)}  \mright]$ for every $v,v^\star$. Fix $g$, and for every round $t$, let $c_t \in \{0,1\}$ such that in round $t$, the correct arm is $(1, c_t)$. For every round $t$ and $v \in [n]$, define $c_{t,v} = 1[e_t(v) = (1,c_t)]$. Now, note that:
    \[
    \frac{p_v(g)}{p_0(g)} = \prod_{t \in [T]} \frac{(1/2)^n (1/2 + (2c_{t,v} - 1) \epsilon)}{(1/2)^{n+1}}
    = \prod_{t \in [T]} (1 + (2c_{t,v} - 1)2\epsilon).
    \]

    Therefore, if $v \neq v^\star$ then:
    \[
    \mathbb{E}_{g \sim P_{(1,v^\star)}^T} \mleft[ \frac{p_v(g)}{p_0(g)}  \mright]
    =
    \prod_{t \in [T]} \mathbb{E}_{c_{t,v} \sim \Ber(1/2)}  \mleft[1 + (2c_{t,v} - 1)2\epsilon \mright] = 1.
    \]
    Otherwise, if $v=v^\star$ then:
    \[
    \mathbb{E}_{g \sim P_{(1,v^\star)}^T} \mleft[ \frac{p_v(g)}{p_0(g)}  \mright]
    =
    \prod_{t \in [T]} \mathbb{E}_{c_{t,v} \sim \Ber(1/2 + \epsilon)}  \mleft[1 + (2c_{t,v} - 1)2\epsilon \mright] = (1 + 4\epsilon^2)^T.
    \]
    Plugging this into the RHS of \eqref{eq:kl-bound} which upper bounds $\KL(P_{mix^T} || P^T_0)$, we obtain:
    \begin{align*}
            \KL(P_{mix^T} || P^T_0)
            &\leq
            \frac{1}{n} \sum_{v^\star \in [n]} \ln \mleft( \frac{n-1 + (1+4\epsilon^2)^T}{n} \mright) \\
            &=  \ln \mleft( 1 + \frac{(1+4\epsilon^2)^T -1}{n} \mright) \\
            & \leq \frac{(1+4\epsilon^2)^T -1}{n},
    \end{align*}
    as claimed.
\end{proof}

We now prove an upper bound on $T(A', S_0)$ for the two-batch case ($k=1$) that holds under the assumption that $A'$ is good for any strategy $S$ from our pool of strategies.

\begin{lemma} \label{lem:one-batch-lower-bound}
    Suppose that $k=1$, $0 < \epsilon \leq 0.1$, and $n \geq 10$. Let $A'$ be good. Then $T(A',S_0) \geq T^\star/2 := \frac{\ln (n/10)}{8 \epsilon^2}$.
\end{lemma}

\begin{proof}
    Let $T$ be a random variable counting the rounds for which $A'$ is being executed. 
    We will show that $P_0^{\infty}[T > T^\star] \geq 1/2$, which implies the claim.
    Let $\hat{J} \in \{0,1\}$ be the output of $A'$. Let $E$ be the event where $A'$ stops after $T \leq T^\star$ many rounds and outputs  $\hat{J} = 0$. 
    By Pinsker's inequality, we have that:
    \begin{align*}
        \mleft \lvert P_0^{T^\star}[E] - P_{mix^{T^\star}}[E] \mright \rvert
        &\leq
        \sqrt{\frac{1}{2}\KL(P_{mix^{T^\star}} \lVert P_0^{T^\star})} \\
        & \leq
        \sqrt{\frac{1}{2}\frac{(1 + 4 \epsilon^2)^{T^\star}-1}{n}} \tag{Lemma~\ref{lem:kl-bound}}\\
        & \leq
        \sqrt{\frac{1}{2}\frac{e^{4 \epsilon^2 T^\star}-1}{n}} \\
        & \leq 1/4.
    \end{align*}
    Note that for $P_\star$ induced by $S_\star$ from our pool, we have $P^\infty_\star (A) = P^{T^\star}_\star(A)$ for any event $A$ contained in the event $T \leq T^\star$. Indeed, determining whether $A$ occurs or not can be done after $T^\star$ draws from $P_\star$. Since $A'$ is good, we also have
    \[
    P_{mix^{T^\star}}[E] = P_{mix-\infty}[E] \leq 0.05. 
    \]
    Again since $A'$ is good, we have:
    \begin{align*}
        1- P_0^{T^\star}[E] &= P_0^{T^\star}[T > T^\star] + P_0^{T^\star}[T \leq T^\star \wedge \hat{J} \neq 0]\\
                        &=
                        P_0^{T^\star}[T > T^\star] + P_0^{\infty}[T \leq T^\star \wedge \hat{J} \neq 0] \\
                        &\leq
                        P_0^{T^\star}[T > T^\star] + 0.05.
    \end{align*}
    Now, note that $P_0^{\infty}[T > T^\star] = P_0^{T^\star}[T > T^\star]$, as the occurrence of the event $T > T^\star$ is already determined after $T^\star$ many rounds.
    Combining this identity with the three inequalities above, we obtain:
    \begin{align*}
        P_0^{\infty}[T > T^\star] &= P_0^{T^\star}[T > T^\star] \\
                              &\geq 0.95 - P_0^{T^\star}[E]\\
                              &\geq 0.95  - 1/4 - P_{mix^{T^\star}}^{T^\star}[E]\\
                              &\geq 0.95 - 1/4 - 0.05 \\
                              &= 0.65.
    \end{align*}
    This concludes the proof.
\end{proof}

\section{The general case (Step 3)} \label{sec:one-to-many}
We now bootstrap the result to the general case, where there are $k+1 \geq 2$ many batches.

\begin{lemma} \label{lem:many-batches-lower-bound}
    Let $A'$ be a good algorithm for SBI. Then $T(A', S_0) \geq k \frac{\ln(n/10)}{20 \epsilon^2}$.
\end{lemma}

\begin{proof}
    Let $T_u(A',S_0)$ be the expected number of rounds where $A'$ pulls $u \in [k]$ against an $S_0$-adversary. That is, $T(A', S_0) = \sum_{u \in [k]} T_u(A',S_0)$ (from linearity of expectation, and since we assume that $A'$ never pulls $0$). Suppose towards contradiction that $T(A', S_0) \leq k \frac{\ln(n/10)}{20 \epsilon^2}$.  Therefore, there exists $u$ so that $T_u(A',S_0) \leq  \frac{\ln(n/10)}{20 \epsilon^2}$. Based on the queries of $A'$ on batch $u$, we will construct a good algorithm $B'$ for SBI for the case $k=1$, which we call the ``small instance". The instance of the problem that $A'$ solves is called the ``big instance". We stress that the adversary is unaware of $A'$, and produces input for the small instance, handled by $B'$. However, $B'$ will simulate input for $A'$, and use its decisions after adapting them to the small instance. The non-$0$ experts in the small instance will be treated as the experts of batch $u$ in the big instance. Expert $0$ in the small instance remains expert $0$ in the big instance. $B'$ operates as follows.
    \begin{enumerate}
        \item In rounds $t= 1,...$
        \begin{enumerate}
            \item $B'$ receives advice for the small instance from the adversary. 
            \item $B'$ generates the advice of all batches of the big instance, except from $0, u$, according to the policy described in Section~\ref{sec:restricted-step0}. It passes to $A'$ the advice of $0, u$ generated by the adversary, and also the self-generated advice for the other batches.
            \item \label{itm:loop-start} $B'$ checks what $A'$ wishes to do.
            \item If $A'$ queries a batch $u' \neq u$: 
            \begin{enumerate}
                \item $B'$ draws the correct arm of $u'$ according to $\Ber(1/2)$ and passes the outcome to $A'$.
                \item $B'$ generates advice for all batches of the big instance according to the policy described in Section~\ref{sec:restricted-step0}, and passes it to $A'$.
                \item Return to Item~\ref{itm:loop-start}.
            \end{enumerate}
            \item Else, if $A'$ queries $u$: pass the correct arm received from the adversary to $A'$, as the correct arm for batch $u$.
            \item If $A'$ stops, then:
            \begin{enumerate}
                \item If $A'$ outputs $u$ or $0$, $B'$ outputs the same output as $A'$ (where $u$ refers to the non-$0$ batch in the small instance).
                \item Else, if $A'$ outputs some batch other than $u$ or $0$, $B'$ outputs batch $0$.
            \end{enumerate}  
        \end{enumerate}
    \end{enumerate}

    Since the adversary must choose a strategy from the pool for the small instance, the big problem instance generated for $A'$ is also from the pool, and note that the special batch in the small and big instances is the same batch: it is one of the two batches of the small instance.
    When $A'$ stops and outputs a batch, $B'$ stops as well. If $A'$ outputs a valid batch for the small instance ($u$ or $0$), $B'$ outputs the same batch, and otherwise it outputs $0$. Since $A'$ is good, it outputs the correct batch with probability at least $0.95$, and therefore $B'$ also outputs the correct batch with probability at least $0.95$. Thus, $B'$ is good for the small instance. However, we also know that $T_u(A',S_0) <  \frac{\ln(n/10)}{20 \epsilon^2}$, and $T_u(A',S_0)$ is the expected number of rounds that $B'$ is executed for against an $S_0$-adversary. This contradicts Lemma~\ref{lem:one-batch-lower-bound}. Therefore, $T(A', S_0) \geq k \frac{\ln(n/10)}{20 \epsilon^2}$.
\end{proof}

\section{Concluding the lower bound for BwE (Step 4)}

\begin{theorem} \label{thm:main}
    Let $N,K$ be positive natural numbers such that $N \geq 2K$. Then, for any $T \geq K \ln (N/K) $, and for any BwE algorithm, there exists an adversarial strategy for which $R_T = \Omega \mleft(\sqrt{TK \log (N/K)} \mright)$.
\end{theorem}

\begin{proof}
    Assume without loss of generality that $N,K$ match our reduced setting ($k:= (K-1)/2 \in \mathbb{N}$ and $n:= (N-1)/k \in \mathbb{N}$), and that $n > 10$.
    Let $T \geq K \ln (N/K)$ and  fix $\epsilon :=  \sqrt{\frac{k \ln (n/10)}{100T}} $. Therefore, $k \frac{\ln(n/10)}{20 \epsilon^2} = 5T$ and $\epsilon \leq 0.1$. Let $A$ be an algorithm for BwE, and suppose it has $R_T \leq \sqrt{T k \ln (n/10)}/100000 := r(T)$. Thus:
    \[
    \frac{1000 r(T)}{\epsilon} = \frac{\sqrt{T k \ln (n/10)}/100 }{\sqrt{\frac{k \ln (n/10)}{100T}}} = T/10 < T.
    \]
    Then, by Lemma~\ref{lem:reduction}, there exists a good SBI algorithm $A'$ with $T(A',S) \leq T = k \frac{\ln(n/10)}{100 \epsilon^2}$ for any strategy $S$ from the pool. This contradicts Lemma~\ref{lem:many-batches-lower-bound}, thus
    \[
     R_T > \sqrt{T k \ln (n/10)}/100000 = \Omega(\sqrt{TK \log (N/K)}),
    \]
    as desired.
\end{proof}

\section{Future work} \label{sec:future}
While in its full generality, the BwE problem allows the adversary to be adaptive, previous lower bounds hold also against an \emph{oblivious} adversary who sets the advice and losses for all rounds in advance. We conjecture that the tight lower bound proved in this work holds against an oblivious adversary as well. More specifically, we conjecture that the oblivious adversary used in \cite{ito2024minimax} achieves the desired lower bound even against an improper learner.

\section*{Acknowledgments}
We thank Emmanuel Esposito for insightful discussions on the works \cite{cesa2025improved, ito2024minimax}.

\bibliographystyle{alphaurl}
\bibliography{bib.bib}

\end{document}